\documentclass{colt2016} % Include author names

\title{Benefits of depth in neural networks}
\usepackage{times}
\coltauthor{% %XXX without this comment line (or putting \Name here), latex inserts a space... lol
  \Name{Matus Telgarsky} \Email{mtelgars@cs.ucsd.edu}\\
  \addr University of Michigan, Ann Arbor
}

\usepackage{thmtools}
\usepackage{hyperref}
\usepackage{cleveref}
\usepackage{mathrsfs}
\usepackage{wrapfig}
\usepackage{mathtools}
\usepackage{nicefrac}

\def\ddefloop#1{\ifx\ddefloop#1\else\ddef{#1}\expandafter\ddefloop\fi}
\def\ddef#1{\expandafter\def\csname b#1\endcsname{\ensuremath{\mathbf{#1}}}}
\ddefloop ABCDEFGHIJKLMNOPQRSTUVWXYZ\ddefloop
\def\ddef#1{\expandafter\def\csname bb#1\endcsname{\ensuremath{\mathbb{#1}}}}
\ddefloop ABCDEFGHIJKLMNOPQRSTUVWXYZ\ddefloop
\def\ddef#1{\expandafter\def\csname c#1\endcsname{\ensuremath{\mathcal{#1}}}}
\ddefloop ABCDEFGHIJKLMNOPQRSTUVWXYZ\ddefloop
\def\ddef#1{\expandafter\def\csname v#1\endcsname{\ensuremath{\boldsymbol{#1}}}}
\ddefloop ABCDEFGHIJKLMNOPQRSTUVWXYZabcdefghijklmnopqrstuvwxyz\ddefloop
\def\ddef#1{\expandafter\def\csname
  v#1\endcsname{\ensuremath{\boldsymbol{\csname #1\endcsname}}}}
\ddefloop
    {alpha}{beta}{gamma}{delta}{epsilon}{varepsilon}{zeta}{eta}{theta}{vartheta}{iota}{kappa}{lambda}{mu}{nu}{xi}{pi}{varpi}{rho}{varrho}{sigma}{varsigma}{tau}{upsilon}{phi}{varphi}{chi}{psi}{omega}{Gamma}{Delta}{Theta}{Lambda}{Xi}{Pi}{Sigma}{varSigma}{Upsilon}{Phi}{Psi}{Omega}\ddefloop

\def\1{\mathbf 1}
\def\R{\mathbb R}

\def\fG{\mathfrak G}

\def\Pr{\textup{Pr}}

\newcommand{\ip}[2]{\left\langle #1, #2 \right \rangle}

\newcommand{\COMMENT}[3]{\marginpar{{\upshape\color{#1}\textbf{[#2]}}}{\upshape\color{#1}[#2: #3]}}

\numberwithin{equation}{section}

\def\srelu{\sigma_{\textsc{r}}}

\def\Cr{\textup{Cr}}

\def\sgn{\textup{sgn}}

\def\phmax{\phi_{\max}}
\def\phmin{\phi_{\min}}

\def\Sh{\textup{Sh}}
\def\VC{\textup{VC}}

\newenvironment{proofof}[1]{\begin{proof}\textbf{(of {#1})}}{\end{proof}}

\newif\ifarxiv
\begin{document}

\maketitle

\begin{abstract}
  For any positive integer $k$,
  there exist neural networks
  with $\Theta(k^3)$ layers,
  $\Theta(1)$ nodes per layer,
  and $\Theta(1)$ distinct parameters
  which can not be approximated by networks with $\cO(k)$ layers
  unless they are exponentially large --- they must possess $\Omega(2^k)$ nodes.
  This result is proved here for a class of nodes termed \emph{semi-algebraic gates} which includes
  the common choices of ReLU, maximum, indicator, and piecewise polynomial functions, therefore establishing
  benefits of depth against not just standard networks with ReLU gates, but also convolutional networks
  with ReLU and maximization gates,
  sum-product networks,
  and boosted decision trees
  (in this last case with a stronger separation: $\Omega(2^{k^3})$ total tree nodes are required).
   \end{abstract}

%\ifarxiv
%\else
%\fi

\section{Setting and main results}
\label{sec:intro}

A neural network is a model of real-valued computation defined by a connected directed graph as follows.
Nodes await real numbers on their incoming edges,
thereafter computing a function of these reals and transmitting it along their outgoing edges.
Root nodes apply their computation to a vector provided as input to the network,
whereas internal nodes apply their computation to the output of other nodes.
Different nodes may compute different functions, two common choices being
the maximization gate $v\mapsto \max_i v_i$ (where $v$ is the vector of values on incoming edges),
and the \emph{standard ReLU gate} $v\mapsto \srelu(\ip{a}{v} + b)$
where $\srelu(z) := \max\{0,z\}$ is called the ReLU
(rectified linear unit), and the parameters $a$ and $b$ may vary from node to node.
Graphs in the present work are acyclic,
and there is exactly one node with no outgoing edges
whose computation is the output of the network.

Neural networks distinguish themselves from many other function classes used in machine learning by
possessing multiple \emph{layers}, meaning the output is the result of composing together an arbitrary
number of (potentially complicated) nonlinear operations;
by contrast, the functions computed by boosted decision stumps and SVMs can be written
as neural networks with a constant number of layers.

The purpose of the present work is to show that standard types of networks always gain in representation
power with the addition of layers.
Concretely: it is shown that for every positive integer $k$,
there exist neural networks
with $\Theta(k^3)$ layers,
$\Theta(1)$ nodes per layer,
and $\Theta(1)$ distinct parameters
which can not be approximated by networks with $\cO(k)$ layers
and $o(2^k)$ nodes.

\subsection{Main result}
\label{sec:intro:main}

Before stating the main result, a few choices and pieces of notation deserve explanation.
First, the target many-layered function uses standard ReLU gates;
this is by no means necessary, and a more general statement can be found in \Cref{fact:main:gen}.
Secondly, the notion of approximation is the $L^1$ distance: given two functions $f$ and $g$,
their pointwise disagreement $|f(x) - g(x)|$ is averaged over the cube $[0,1]^d$.
Here as well, the same proofs allow flexibility (cf. \Cref{fact:main:gen}).
Lastly, the shallower networks used for approximation use \emph{semi-algebraic gates},
which generalize the earlier maximization and standard ReLU gates,
and allow for analysis of not just standard networks with ReLU gates,
but convolutional networks with ReLU and maximization gates \citep{imagenet_sutskever},
sum-product networks (where nodes compute polynomials) \citep{poon_domingos_spn},
and boosted decision trees;
the full definition of semi-algebraic gates appears in \Cref{sec:sa}.

\begin{theorem}  \label{fact:main}
  Let any integer $k \geq 1$ and any dimension $d\geq 1$ be given.
  There exists $f:\R^d \to \R$ computed by a neural network with standard ReLU gates
  in $2k^3+8$ layers, $3k^3 + 12$ total nodes, and $4+d$ distinct parameters
  so that
  \[
    \inf_{g \in \cC} \int_{[0,1]^d} |f(x)-g(x)|dx \geq \frac 1 {64},
  \]
  where $\cC$ is the union of the following two sets of functions.
  \begin{itemize}
    \item
      Functions computed by networks of $(t,\alpha,\beta)$-semi-algebraic gates
      in $\leq k$ layers and $\leq 2^{k}/(t\alpha\beta)$ nodes.
      (E.g., as with standard ReLU networks
      or with convolutional neural networks with standard ReLU and maximization gates; cf. \Cref{sec:sa}.)
    \item
      Functions
      computed by linear combinations of $\leq t$ decision trees
      each with $\leq 2^{k^3}/t$ nodes.
      (E.g., the function class used by boosted decision trees;
      cf. \Cref{sec:sa}.)
  \end{itemize}
\end{theorem}

Analogs to \Cref{fact:main} for boolean circuits --- which have boolean inputs
routed through $\{\textup{and},\textup{or},\textup{not}\}$ gates --- have been studied extensively
by the circuit complexity community, where they are called \emph{depth hierarchy theorems}.
The seminal result, due to \citet{hastad_thesis}, establishes the inapproximability
of the parity function by shallow circuits (unless their size is exponential).
Standard neural networks appear to have received less study;
closest to the present work is an investigation by
\citet{ohad_nn_apx} analyzing the case $k=2$ when the dimension $d$ is large,
showing an exponential separation between 2- and 3-layer networks, a regime not handled by \Cref{fact:main}.
Further bibliographic notes and open problems may be found in \Cref{sec:bib}.

The proof of \Cref{fact:main} (and of the more general \Cref{fact:main:gen}) occupies \Cref{sec:apx}.
The key idea is that just a few function compositions (layers) suffice to construct a highly oscillatory function,
whereas function addition (adding nodes but keeping depth fixed) gives a function with few oscillations.
Thereafter, an elementary counting argument suffices to show that low-oscillation functions can not approximate
high-oscillation functions.

\subsection{Companion results}
\label{sec:intro:companion}

\Cref{fact:main} only provides the existence of \emph{one}
network (for each $k$) which can not be approximated by
a network with many fewer layers.  It is natural to wonder if there are \emph{many} such
special functions.  The following bound indicates their population is in fact quite modest.

Specifically, the construction behind \Cref{fact:main}, as elaborated in \Cref{fact:main:gen},
can be seen as exhibiting $\cO(2^{k^3})$ points, and a fixed labeling of these points, upon which a shallow network hardly improves
upon random guessing.  The forthcoming \Cref{fact:sa_fit_few} similarly shows that even on the more simpler task of fitting $\cO(k^9)$ points,
the earlier class of networks is useless on most random labellings.

In order to state the result, a few more definitions are in order.
Firstly, for this result, the notion of neural network is more restrictive.
Let a \emph{neural net graph $\fG$} denote not only the graph structure (nodes and edges),
but also an assignment of gate functions to nodes, of edges to the inputs of gates,
and an assignment of free parameters $w\in\R^p$ to the parameters of the gates.
Let $\cN(\fG)$ denote the class of functions obtained by varying the free parameters;
this definition is fairly standard, and is discussed in more detail in \Cref{sec:sa}.
As a final piece of notation, given a function $f:\R^d\to\R$, let $\tilde f : \R^d \to \{0,1\}$ denote
the corresponding classifier $\tilde f(x) := \1[f(x) \geq 1/2]$.

\begin{theorem}
  \label{fact:sa_fit_few}
  Let any neural net graph $\fG$ be given with $\leq p$ parameters in $\leq l$ layers and $\leq m$ total
  $(t,\alpha,\beta)$-semi-algebraic nodes.
  Then for any $\delta > 0$
  and any
  $n \geq 8pl^2 \ln(8emt \alpha\beta p(l+1)) + 4\ln(1/\delta)$ points $(x_i)_{i=1}^n$,
  with probability $\geq 1-\delta$ over uniform random labels $(y_i)_{i=1}^n$,
  \[
    \inf_{f\in\cN(\fG)}  \frac 1 n \sum_{i=1}^n \1[\tilde f(x_i) \neq y_i]
    \geq \frac 1 4.
  \]
\end{theorem}

This proof is a direct corollary of the VC dimension of semi-algebraic networks,
which in turn can be proved by a small modification of the VC dimension proof for
piecewise polynomial networks \citep[Theorem 8.8]{anthony_bartlett_nn}.
Moreover, the core methodology for VC
dimension bounds of neural networks is due to \citeauthor{warren},
whose goal was an analog of \Cref{fact:sa_fit_few} for polynomials \citep[Theorem 7]{warren}.

\ifarxiv
\begin{lemma}[name={Simplification of \Cref{fact:pp_vc}}]
\else
\begin{lemma}[Simplification of \Cref{fact:pp_vc}]
\fi
  \label[lemma]{fact:vc_sa}
  Let any neural net graph $\fG$ be given with $\leq p$ parameters in $\leq l$ layers and $\leq m$ total nodes,
  each of which is $(t,\alpha,\beta)$-semi-algebraic.
  Then
  \[
    \VC(\cN(\fG)) \leq
    6p(l+1)\big(
      \ln(2p(l+1)) + \ln(8emt\alpha) + l\ln(\beta)
    \big).
  \]
\end{lemma}

The proof of \Cref{fact:sa_fit_few} and \Cref{fact:vc_sa} may be found in \Cref{sec:vc}.
The argument for the VC dimension is very close to the argument for \Cref{fact:main}
that a network with few layers has few oscillations; see \Cref{sec:vc} for further discussion
of this relationship.

\section{Semi-algebraic gates and assorted network notation}
\label{sec:sa}

The definition of a semi-algebraic gate is unfortunately complicated;
it is designed to capture a few standard nodes in a single abstraction
without degrading the bounds.
Note that the name \emph{semi-algebraic set} is standard \citep[Definition 2.1.4]{bochnak_coste_roy},
and refers to a set defined by unions and intersections of polynomial inequalities
(and thus the name is somewhat abused here).

\begin{definition}
  A function $f:\R^k\to\R$ is \emph{$(t,\alpha,\beta)$-sa ($(t,\alpha,\beta)$-semi-algebraic)}
  if there exist $t$ polynomials $(q_i)_{i=1}^t$ of degree $\leq\alpha$,
  and $m$ triples $(U_j,L_j,p_j)_{j=1}^m$ where $U_j$ and $L_j$ are subsets
  of $[t]$ (where $[t] := \{1,\ldots,t\}$) and $p_j$ is a polynomial of degree $\leq\beta$, such that
  \[
    f(v) = \sum_{j=1}^m p_j(v)
    \left(\prod_{i \in L_j} \1[q_i(v) < 0]\right)
    \left(\prod_{i \in U_j} \1[q_i(v) \geq 0]\right).
  \]
\end{definition}

A notable trait of the definition is that the number of terms $m$ does not need to enter
the name as it does not affect any of the complexity estimates herein (e.g., \Cref{fact:main} or \Cref{fact:sa_fit_few}).

Distinguished special cases of semi-algebraic gates are as follows in \Cref{fact:sa:ex}.
The standard piecewise polynomial gates generalize the ReLU and have received a fair bit of attention
in the theoretical community \citep[Chapter 8]{anthony_bartlett_nn};
here a function $\sigma:\R\to\R$ is \emph{$(t,\alpha)$-poly}
if $\R$ can be partitioned into $\leq t$ intervals
so that $\sigma$ is a polynomial of degree $\leq \alpha$ within each piece.
The maximization and minimization gates have become popular due to their use in convolutional
networks \citep{imagenet_sutskever}, which will be discussed more in \Cref{sec:sa:nn}.
Lastly, decision trees and boosted decision trees are practically successful classes usually viewed as competitors
to neural networks \citep{caruana_empirical}, and have the following structure.

\begin{definition}
  A \emph{$k$-dt (decision tree with $k$ nodes)} is defined recursively as follows.
  If $k=1$, it is a constant function.
  If $k > 1$, it first evaluates $x\mapsto \1[\ip{a}{x} - b \geq 0]$,
  and thereafter conditionally evaluates either
  a left $l$-dt or a right $r$-dt where $l+r<k$.
  A \emph{$(t,k)$-bdt} (boosted decision tree)
  evaluates $x\mapsto \sum_{i=1}^t c_i g_i(x)$ where each $c_i\in \R$ and each $g_i$ is a $k$-dt.
\end{definition}

\begin{lemma}[Example semi-algebraic gates]
  \label[lemma]{fact:sa:ex}
  \begin{enumerate}
    \item
      If $\sigma :\R\to\R$ is $(t,\beta)$-poly and $q:\R^d\to\R$ is a polynomial of degree $\alpha$,
      then the standard piecewise polynomial gate $\sigma\circ q$ is $(t,\alpha,\alpha\beta)$-sa.
      In particular, the standard ReLU gate $v\mapsto \srelu(\ip{a}{v} + b)$ is $(1,1,1)$-sa.
    \item
      Given polynomials $(p_i)_{i=1}^r$ of degree $\leq\alpha$,
      the standard $(r,\alpha)$-min and -max gates $\phmin(v) := \min_{i\in[r]} p_i(v)$
      and $\phmax(v) := \max_{i\in[r]} q_i(v)$
      are $(r(r-1), \alpha, \alpha)$-sa.
    \item
      Every $k$-dt is $(k,1,0)$-sa,
      and every $(t,k)$-bdt is $(tk,1,0)$.
  \end{enumerate}
\end{lemma}
The proof of \Cref{fact:sa:ex} is mostly a matter of unwrapping definitions, and is deferred to
\Cref{sec:proofs}.
Perhaps the only interesting encoding is for the maximization gate (and similarly the minimization gate),
which uses $\max_{i} v_i = \sum_i v_i (\prod_{j < i} \1[v_i > v_j])(\prod_{j > i} \1[v_i \geq v_j])$.

\subsection{Notation for neural networks}
\label{sec:sa:nn}

A semi-algebraic gate is simply a function from some domain to $\R$,
but its role in a neural network is more complicated as the domain of the function must be
partitioned into arguments of three types: the input $x\in\R^d$ to the network, the parameter vector $w\in\R^p$,
and a vector of real numbers coming from parent nodes.

As a convention, the input $x\in\R^d$ is only accessed by the root nodes (otherwise ``layer'' has no meaning).
For convenience, let layer 0 denote the input itself: $d$ nodes where node $i$ is the map $x\mapsto x_i$.
The parameter vector $w\in\R^p$ will be made available to all nodes in layers above 0, though they might only
use a subset of it.  Specifically,
an internal node computes a function $f:\R^p\times\R^d \to \R$ using parents $(f_1,\ldots,f_k)$
and a semi-algebraic gate $\phi : \R^p \times \R^k \to \R$, meaning
$f(w,x) := \phi(w_1,\ldots,w_p, f_1(w,x), \ldots, f_k(w,x))$.
Another common practice is to have nodes apply a univariate \emph{activation function} to
an affine mapping of their parents (as with piecewise polynomial gates in \Cref{fact:sa:ex}),
where the weights in the affine combination are the parameters to the network, and additionally
correspond to edges in the graph.
It is permitted for the same parameter to appear multiple times in a network, which explains how
the number of parameters in \Cref{fact:main} can be less than the number of edges and nodes.
The entire network computes some function $F_\fG:\R^p\times\R^d\to \R$, which is equivalent to the
function computed by the single node with no outgoing edges.

As stated previously, $\fG$ will denote not just the graph (nodes and edges) underlying a network,
but also an assignment of gates to nodes, and how parameters and parent outputs are plugged into
the gates (i.e., in the preceding paragraph, how to write $f$ via $\phi$).
$\cN(\fG)$ is the set of functions obtained by varying $w\in\R^p$,
and thus $\cN(\fG) := \{ F_\fG(w,\cdot) : w\in\R^p\}$ where $F_\fG$ is the function defined as above,
corresponding to computation performed by $\fG$.
The results related to VC dimension, meaning \Cref{fact:sa_fit_few} and \Cref{fact:vc_sa},
will use the class $\cN(\fG)$.

Some of the results, for instance \Cref{fact:main} and its generalization \Cref{fact:main:gen},
will let not only the parameters but also network graph $\fG$ vary.
Let $\cN_d((m_i, t_i,\alpha_i,\beta_i)_{i=1}^l)$ denote a network where layer $i$ has $\leq m_i$ nodes
where each is $(t_i,\alpha_i,\beta_i)$-sa and the input has dimension $d$.
As a simplification, let $\cN_d(m,l, t,\alpha,\beta)$ denote networks of $(t,\alpha,\beta)$-sa gates in $\leq l$ layers (not including layer 0)
each with $\leq m$ nodes.
There are various empirical prescriptions on how to vary the number of nodes per layer;
for instance, convolutional networks typically have an increase between layer 0 and layer 1,
followed by exponential decrease for a few layers, and finally a few layers with the same number of nodes
\citep{fukushima_convnet,lecun_convnet,imagenet_sutskever}.

\section{Benefits of depth}
\label{sec:apx}

The purpose of this section is to prove \Cref{fact:main} and its generalization \Cref{fact:main:gen}
in the following three steps.

\begin{enumerate}
  \item
    Functions with few oscillations poorly approximate functions with many oscillations.

  \item
    Functions computed by networks with few layers must have few oscillations.

  \item
    Functions computed by networks with many layers can have many oscillations.
\end{enumerate}

\subsection{Approximation via oscillation counting}

\begin{wrapfigure}{R}{0.4\textwidth}
  \vspace{-8pt}

  \includegraphics[width=0.4\textwidth]{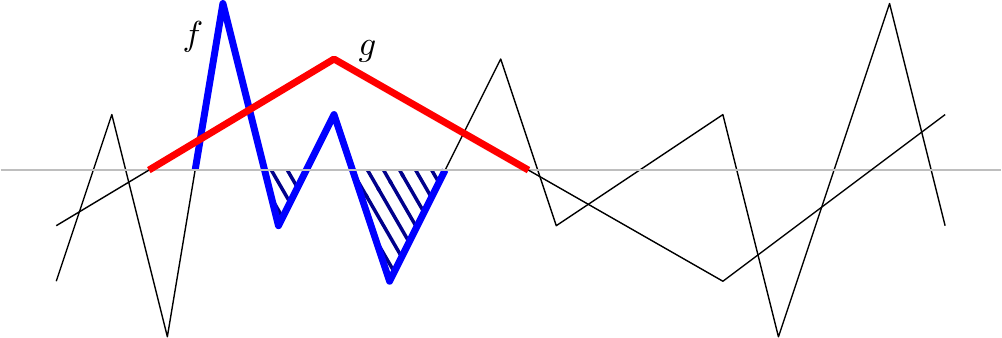}

  \vspace{0pt}

    \caption{$f$ crosses more than $g$.}
  \label{fig:apx:1}
  \vspace{-8pt}
\end{wrapfigure}
The idea behind this first step is depicted at right.
Given functions $f:\R\to\R$ and $g:\R\to\R$ (the multivariate case will come soon),
let $\cI_f$ and $\cI_g$ denote partitions of $\R$ into intervals so
that the classifiers $\tilde f(x) =\1[f(x)\geq 1/2]$ and $\tilde g$ are constant
within each interval.
To formally count oscillations, define
\emph{the crossing number $\Cr(f)$ of $f$} as $\Cr(f) = |\cI_f|$ (thus $\Cr(\srelu) = 2$).
If $\Cr(f)$ is much larger than $\Cr(g)$,
then most piecewise constant regions of $\tilde g$ will exhibit many oscillations of $f$,
and thus $g$ poorly approximates $f$.

\begin{lemma}
  \label[lemma]{fact:mono:crossing_lb:2}
  Let $f:\R\to\R$ and $g:\R\to\R$ be given,
  and take $\cI_f$ to denote the partition of $\R$ given by the pieces of $\tilde f$
  (meaning $|\cI_f| = \Cr(f)$).
  Then
  \[
    \frac 1 {\Cr(f)} \sum_{U \in \cI_f} \1[\forall x \in U \centerdot \tilde f(x) \neq \tilde g(x) ]
    \geq \frac 1 2\left(1 - 2 \left(\frac {\Cr(g)}{\Cr(f)}\right)\right).
  \]
\end{lemma}

The arguably strange form of the left hand side of the bound in
\Cref{fact:mono:crossing_lb:2} is to accommodate different notions of distance.
For the $L^1$ distance with the Lebesgue measure as in \Cref{fact:main},
it does not suffice for $f$ to cross 1/2: it must be \emph{regular},
meaning
it must cross by an appreciable distance, and the crossings must be evenly spaced.
(It is worth highlighting that the ReLU easily gives rise to a regular $f$.)
However, to merely show that $f$ and $g$ give very different classifiers $\tilde f$ and $\tilde g$
over an arbitrary measure (as in part of \Cref{fact:main:gen}), no additional regularity is needed.

\begin{proofof}{\Cref{fact:mono:crossing_lb:2}}
  Let $\cI_f$ and $\cI_g$ respectively denote the sets of intervals
  corresponding to $\tilde f$ and $\tilde g$,
  and set $s_f := \Cr(f) = |\cI_f|$ and $s_g := \Cr(g) = |\cI_g|$.

  For every $J\in \cI_g$, set $X_J := \{ U \in \cI_f : U \subseteq J \}$.
  Fixing any $J\in \cI_g$,
  since $\tilde g$ is constant on $J$ whereas $\tilde f$ alternates,
  the number of elements in $X_J$ where $\tilde g$ disagrees everywhere with $\tilde f$
  is $|X_J|/2$ when $|X_J|$ is even
  and at least $(|X_J|-1)/2$ when $|X_J|$ is odd,
  thus at least $(|X_J|-1)/2$ in general.
  As such,
  \begin{align}
    \frac 1 {s_f} \sum_{U \in \cI_f} \1[\forall x \in U \centerdot \tilde f(x) \neq \tilde g(x) ]
    &
    \geq
    \frac 1 {s_f} \sum_{J \in \cI_{g}} \sum_{U \in X_J}
    \1[\forall x \in U \centerdot \tilde f(x) \neq \tilde g(x) ]
   \geq
    \frac 1 {s_f} \sum_{J \in \cI_{g}} \frac {|X_J| - 1} {2}\label{eq:crossing_lb:1}.
  \end{align}
  To control this expression, note that every $X_J$ is disjoint, however $X:=\cup_{J\in\cI_j} X_j$
  can be smaller than $\cI_f$: in particular, it misses intervals $U\in \cI_f$
  whose interior intersects with the boundary of an interval in $\cI_g$.
  Since there are at most $s_g-1$ such boundaries,
  \[
    s_f = |\cI_f| \leq s_g - 1 + |X| \leq s_g + \sum_{J\in\cI_g} |X_J|,
  \]
  which rearranges to gives $\sum_{J\in\cI_g} |X_J| \geq s_f - s_g$.
  Combining this with \cref{eq:crossing_lb:1},
  \[
    \frac 1 {s_f} \sum_{U \in \cI_f} \1[\forall x \in U \centerdot \tilde f(x) \neq \tilde g(x) ]
    \geq \frac {1}{2s_f}\left( s_f - s_g - s_g \right)
    = \frac 1 2 \left(1 - \frac {2s_g}{s_f}\right).
  \]
\end{proofof}

\subsection{Few layers, few oscillations}
\label{sec:apx:lb}

As in the preceding section, oscillations of a function $f$ will be counted via the crossing number $\Cr(f)$.
Since $\Cr(\cdot)$ only handles univariate functions, the multivariate case is handled by first choosing an affine
map $h : \R\to\R^d$ (meaning $h(z) = az + b$) and considering $\Cr(f\circ h)$.

Before giving the central upper bounds and sketching their proofs, notice by analogy to polynomials how
compositions and additions vary in their impact upon oscillations.  By adding together two polynomials,
the resulting polynomial has at most twice as many terms and does not exceed the maximum degree of either polynomial.
On the other hand, composing polynomials, the result has the product of the degrees and can have more than the product
of the terms.  As both of these can impact the number of
roots or crossings (e.g., by the Bezout Theorem or Descartes' Rule of Signs),
composition wins the race to higher oscillations.

\begin{lemma}
  \label[lemma]{fact:sa_crossing}
  Let $h:\R\to\R^d$ be affine.
  \begin{enumerate}
    \item
      Suppose $f \in \cN_d((m_i,t_i,\alpha_i,\beta_i)_{i=1}^l)$
      with $\min_i\min\{\alpha_i,\beta_i\} \geq 1$.
                      Setting $\alpha := \max_i \alpha_i, \beta := \max_i \beta_i$, $t := \max_i t_i$,
      $m := \sum_i m_i$, then
            $\Cr(f\circ h) \leq 2(2tm\alpha/l)^l \beta^{l^2}$.
    \item
            Let $k$-dt $f:\R^d\to\R$ and $(t,k)$-bdt $g:\R^d\to\R$ be given.             Then $\Cr(f\circ h)\leq k$ and $\Cr(g\circ h) \leq 2tk$.
  \end{enumerate}
\end{lemma}

\Cref{fact:sa_crossing} shows the key tradeoff: the number of layers is in the exponent,
while the number of nodes is in the base.

Rather than directly controlling $\Cr(f\circ h)$, the proofs will first show $f\circ h$ is $(t,\alpha)$-poly,
which immediately bounds $\Cr(f\circ h)$ as follows.
\begin{lemma}
  \label[lemma]{fact:poly_cr}
  If $f:\R\to\R$ is $(t,\alpha)$-poly,
  then $\Cr(f) \leq t(1+\alpha)$.
\end{lemma}
\begin{proof}
  The polynomial in each piece has at most $\alpha$ roots,
  which thus divides each piece into $\leq 1+\alpha$ further pieces
  within which $\tilde f$ is constant.
\end{proof}

A second technical lemma is needed to reason about combinations
of partitions defined by $(t,\alpha,\beta)$-sa and $(t,\alpha)$-poly functions.

\begin{lemma}
  \label[lemma]{fact:partition_combination}
  Let $k$ partitions $(A_i)_{i=1}^k$ of $\R$ each into at most $t$ intervals be given,
  and set $A := \cup_i A_i$.
  Then there exists a partition $B$ of $\R$ of size at most $kt$
  so that every interval expressible as a union of intersections of elements of $A$
  is a union of elements of $B$.
\end{lemma}

\begin{wrapfigure}{R}{0.3\textwidth}
  \vspace{-8pt}

  \includegraphics[width=0.3\textwidth]{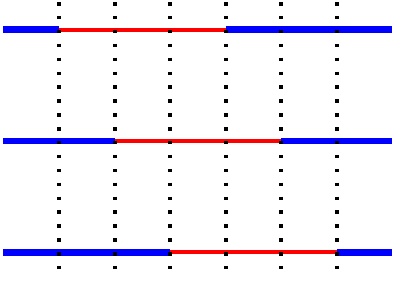}

  \vspace{0pt}

  \caption{Three partitions.}
  \label{fig:partition:counting}
  \vspace{-8pt}
\end{wrapfigure}
The proof is somewhat painful owing to the fact that there is no convention on the structure
of the intervals in the partitions, namely which ends are closed and which are open,
and is thus deferred to \Cref{sec:proofs}.
The principle of the proof is elementary, and is depicted at right:
given a collection of partitions, an intersection of constituent intervals must share endpoints
with intervals in in the intersection, thus the total number of intervals bounds the total number
of possible intersections.
Arguably, this failure to increase complexity in the face of arbitrary intersections
is why semi-algebraic gates do not care about the number of terms in their definition.

Recall that $(t,\alpha,\beta)$-sa means there is a set of $t$ polynomials of degree at most $\alpha$
which form the regions defining the function by intersecting simpler regions
$x\mapsto \1[q(x) \geq 0]$ and $x\mapsto \1[q(x) < 0]$.
As such, in order to analyze semi-algebraic gates composed with piecewise polynomial gates,
consider first the behavior of these predicate polynomials.

\begin{lemma}
  \label[lemma]{fact:poly_circ_pp}
  Suppose $f:\R^k\to\R$ is polynomial with degree $\leq\alpha$
  and $(g_i)_{i=1}^k$ are each $(t,\gamma)$-poly.
  Then $h(x) := f(g_1(x),\ldots,g_k(x))$ is $(tk,\alpha\gamma)$-poly,
  and the partition defining $h$ is a refinement of the partitions for each $g_i$
  (in particular, each $g_i$ is a fixed polynomial (of degree $\leq\gamma$)
  within the $\leq tk$ pieces defining $h$).
\end{lemma}
\begin{proof}
  By \Cref{fact:partition_combination},
  there exists a partition of $\R$ into $\leq tk$ intervals
  which refines the partitions defining each $g_i$.
  Since $f$ is a polynomial with degree $\leq \alpha$,
  then within each of these intervals, its composition with $(g_1,\ldots,g_k)$
  gives a polynomial of degree $\leq \alpha\gamma$.
  \end{proof}

This gives the following complexity bound for composing $(s,\alpha,\beta)$-sa and $(t,\gamma)$-poly gates.

\begin{lemma}
  \label[lemma]{fact:sa_to_poly}
  Suppose $f : \R^k\to\R$ is $(s,\alpha,\beta)$-sa
  and $(g_1,\ldots,g_k)$ are $(t,\gamma)$-poly.
  Then $h(x) := f(g_1(x),\ldots,g_k(x))$ is $(stk(1 +  \alpha\gamma), \beta\gamma)$-poly.
\end{lemma}
\begin{proof}
  By definition, $f$ is polynomial in regions defined by intersections of the predicates
  $U_i(x) = \1[q_i(x) \geq 0]$ and $L_i(x) = \1[q_i(x) < 0]$.
  By \Cref{fact:poly_circ_pp}, $q_i(g_1,\ldots,g_k)$ is $(tk,\alpha\gamma)$-poly,
  thus $U_i$ and $L_i$ together define a partition of $\R$ which has $\Cr(x\mapsto q_i(g_1(x),\ldots,g_k(x)))$
  pieces, which by \Cref{fact:poly_cr} has cardinality at most $tk(1 + \alpha\gamma)$
  and refines the partitions for each $g_i$.
  By \Cref{fact:partition_combination}, these partitions across all predicate polynomials $(q_i)_{i=1}^s$
  can be refined into a single partition of size
  $\leq stk(1 + \alpha\gamma)$,
  and which thus also refines the partitions defined by $(g_1,\ldots,g_k)$.
  Thanks to these refinements, $h$ over any element $U$ of this final partition
  is a fixed polynomial $p_U(g_1,\ldots,g_k)$ of degree $\leq\beta\gamma$,
  meaning $h$ is $(stk(1 + \alpha\gamma), \beta\gamma)$-poly.
\end{proof}

The proof of \Cref{fact:sa_crossing} now follows by \Cref{fact:sa_to_poly}.
In particular, for semi-algebraic networks, the proof is an induction over layers,
establishing node $j$ is $(t_j,\alpha_j)$-poly (for appropriate $(t_j,\alpha_j)$).

\subsection{Many layers, many oscillations}
\label{sec:triangle}

The idea behind this construction is as follows.  Consider any continuous function $f:[0,1]\to[0,1]$ which is a generalization of
a triangle wave with a single peak: $f(0) = f(1) = 0$,
and there is some $a\in(0,1)$ with $f(a) = 1$, and additionally $f$ strictly increases along $[0,a]$ and strictly decreases along
$[a,1]$.

Now consider the effect of the composition $f\circ f = f^2$. Along $[0,a]$, this is a stretched copy of $f$, since $f(f(a)) = f(1) = 0 = f(0) = f(f(0))$
and moreover $f$ is a bijection between $[0,a]$ and $[0,1]$ (when restricted to $[0,a]$).  The same reasoning applies to $f^2$ along $[a,1]$,
meaning $f^2$ is a function with two peaks.  Iterating this argument implies $f^{k}$ is a function with $2^{k-1}$ peaks; the following
definition and lemmas formalize this reasoning.

\begin{definition}
  $f$ is \emph{$(t,[a,b])$-triangle} when it is continuous along $[a,b]$, and $[a,b]$ may
  be divided into $2t$ intervals $[a_i, a_{i+1}]$ with $a_1 = a$ and $a_{2t+1}=b$,
  $f(a_i) = f(a_{i+2})$ whenever $1 \leq i \leq 2t-1$,
  $f(a_1) = 0$,
  $f(a_2) = 1$,
  $f$ is strictly increasing along odd-numbered intervals (those starting from $a_i$ with $i$ odd),
  and strictly decreasing along even-numbered intervals.
\end{definition}

\begin{lemma}
  \label[lemma]{fact:mono:gen:ub:1}
  If $f$ is $(s,[0,1])$-triangle and $g$ is $(t,[0,1])$-triangle,
  then $f\circ g$ is $(2st, [0,1])$-triangle.
\end{lemma}
\begin{proof}
  \def\fab{f_{[a,b]}}
  Since $g([0,1]) = [0,1]$ and $f$ and $g$ are continuous along $[0,1]$,
  then $f\circ g$ is continuous along $[0,1]$.
  In the remaining analysis,
  let $(a_1,\ldots,a_{2s+1})$ and $(c_1,\ldots,c_{2t+1})$
  respectively denote the interval boundaries for $f$ and $g$.

  Now consider any interval $[c_j, c_{j+1}]$ where $j$ is odd,
  meaning the restriction $g_j : [c_j, c_{j+1}]\to [0,1]$ of $g$ to $[c_j, c_{j+1}]$ is strictly increasing.
  It will be shown that $f\circ g_j$ is $(s,[c_j,c_{j+1}])$-triangle,
  and an analogous proof holds for the strictly decreasing restriction
  $g_{j+1} : [c_{j+1}, c_{j+2}] \to [0,1]$,
  whereby it follows that $f\circ g$ is $(2st,[0,1])$ by considering all choices of $j$.

  To this end, note for any $i \in \{1,\ldots, 2s+1\}$ that $g_{j}^{-1}(a_i)$ exists and is unique,
  thus set $a'_i := g_j^{-1}(a_i)$.
  By this choice,
  for odd $i$ it holds that $f(g_j(a'_i)) = f(g_j(g_j^{-1}(a_i))) = f(a_i) = f(a_1) = 0$
  and $f\circ g_j$ is strictly increasing
  along $[a'_i, a'_{i+1}]$ (since $g_j$ is strictly increasing everywhere
  and $f$ is strictly increasing along $[g_j(a'_i), g_j(a'_{i+1})] = [a_i, a_{i+1}]$),
  and similarly even $i$ has $f(g_j(a'_i)) = f(a_2) = 1$
  and $f\circ g_j$ is strictly decreasing along $[a'_i, a'_{i+1}]$.
\end{proof}

\begin{corollary}
  \label[corollary]{fact:mono:gen:ub:2}
  If $f\in \cN_1(m,l,t,\alpha,\beta)$ is $(t,[0,1])$-triangle with $p$ distinct parameters,
  then $f^k \in\cN_1(m,kl,t,\alpha,\beta)$ is $(2^{k-1}t^{k},[0,1])$-triangle with $p$ distinct parameters
  and $\Cr(f^k) = (2t)^k+1$.
\end{corollary}
\begin{proof}
  It suffices to perform $k-1$ applications of \Cref{fact:mono:gen:ub:1}.
\end{proof}

Next, note the following examples of triangle functions.

\begin{lemma}
  \label[lemma]{fact:triangles}
  The following functions are $(1,[0,1])$-triangle.
  \begin{enumerate}
    \item
      $f(z) := \srelu(2\srelu(z) - 4\srelu(z-1/2)) \in \cN_1(2,1, 1,1,1)$.

    \item
      $g(z) := \min\{ \srelu(2z), \srelu(2-2z) \} \in \cN_1(2,1, 2,1,1)$.
    \item
      $h(z) := 4z(1-z) \in \cN_1(1,1, 0, 2,0)$.
      Cf. \citet{schmitt_4x1x_nn}.
  \end{enumerate}
\end{lemma}

Lastly, consider the first example $f(z) = \srelu(2\srelu(z) - 4(\srelu(z-1/2))) = \min\{\srelu(2z), \srelu(2-2z)\}$,
whose graph linearly interpolates (in $\R^2$) between $(0,0)$, $(1/2,1)$, and $(1,0)$.
Consequently, $f\circ f$ along $[0,1/2]$ linear interpolates between $(0,0)$, $(1/4,1)$, and $(1/2,1)$, and
$f\circ f$ is analogous on $[1/2,1]$, meaning it has produced two copies of $f$ and then shrunken them horizontally
by a factor of 2.  This process repeats, meaning $f^k$ has $2^{k-1}$ copies of $f$, and grants the regularity
needed to use the Lebesgue measure in \Cref{fact:main}.

\begin{lemma}
  \label[lemma]{fact:triangle:relu:shape}
  Set $f(z) := \srelu(2\srelu(z) - 4\srelu(z-1/2)) \in \cN_1(2,1,1,1,1)$ (cf. \Cref{fact:triangles}).
  Let real $z \in [0,1]$ and positive integer $k$ be given,
  and choose the unique nonnegative integer $i_k \in \{0,\ldots,2^{k-1}\}$ and real $z_k \in [0, 1)$
  so that $z = (i_k + z_k)2^{1-k}$.
  Then
  \[
    f^k(z) = \begin{cases}
      2z_k
      &\textup{when $0 \leq z_k \leq 1/2$},
      \\
      2(1 - z_k)
      &\textup{when $1/2 < z_k < 1$}.
    \end{cases}
  \]
\end{lemma}

\subsection{Proof of \Cref{fact:main}}

The proof of \Cref{fact:main} now follows:
\Cref{fact:triangle:relu:shape} shows that a many-layered ReLU network can give rise to
a highly oscillatory and regular function $f^k$,
\Cref{fact:sa_crossing} shows that few-layered networks and (boosted) decision trees
give rise to functions with few oscillations,
and lastly \Cref{fact:mono:crossing_lb:2} shows how to combine these into an inapproximability result.

In this last piece, the proof averages over the possible offsets $y\in\R^{d-1}$ and considers univariate
problems after composing networks with the affine map $h_y(z) := (z,y)$.  In this way, the result
carries some resemblance to the random projection technique used in depth hierarchy theorems
for boolean functions \citep{hastad_thesis,rocco_apx},
as well as earlier techniques on complexities of multivariate sets \citep{vitushkin_1,vitushkin_2},
albeit in an extremely primitive form (considering variations along only one dimension).

\begin{proofof}{\Cref{fact:main}}
  Set $h(z) := \srelu(2\srelu(z) - 4\srelu(z-1/2))$ (cf. \Cref{fact:triangles}),
  and define $f_0(z) := h^{k^3+4}(z)$ and $f:\R^d\to\R$ as $f(x) = f_0(x_1)$.
  Let $\cI_f$ denote the pieces of $\tilde f_0$,
  meaning $|\cI_f| = \Cr(f_0)$,
  and \Cref{fact:mono:gen:ub:2} grants $\Cr(f_0) = 2^{k^3+4}+1$.
  Moreover, by \Cref{fact:triangle:relu:shape},
  for any $U\in \cI_f$, $f_0-1/2$ is a triangle with height 1/2
  and base either $2^{-k-1}$ (when $0\in U$ or $1\in U$)
  or $2^{-k}$, whereby
  $\int_U |f_0(x) - 1/2|dx \geq 2^{-k-1}/4 \geq |\cI_f|/16$
  (which has thus made use of the special regularity of $h$).

  Now for any $y\in\R^{d-1}$ define the map $p_y : \R\to\R^d$ as $p_y(z) := (z,y)$.
  If $g$ is a semi-algebraic network with $\leq k$ layers and $m\leq 2^k/(t\alpha\beta)$ total nodes,
  then \Cref{fact:sa_crossing} grants
  $\Cr(g\circ p_y) \leq 2 (2tm\alpha/k)^k \beta^{k^2} \leq 4(tm\alpha\beta)^{k^2} \leq 2^{k^3+2}$.
  Otherwise, $g$ is $(t,2^{k^3}/t)$-bdt,
  whereby \Cref{fact:sa_crossing} gives $\Cr(g\circ p_y) \leq 2t2^{k^3}/t \leq 2^{k^3+2}$ once again.

  By \Cref{fact:mono:crossing_lb:2}, for any $y\in\R^{d-1}$, $\Cr(f\circ p_y) = \Cr(f_0)$, and
  \begin{align*}
        \int_{[0,1]} |f(p_y(z)) - g(p_y(z))| dz
    &= \sum_{U \in \cI_f} \int_U |(f\circ p_y)(z) - (g\circ p_y)(z)| dz
    \\
    &\geq \sum_{U \in \cI_f} \int_U |(f\circ p_y)(z) - 1/2|
    \1[\forall z\in U\centerdot \widetilde{(f\circ p_y)}(z) \neq \widetilde{(g\circ p_y)}(z)] dz
    \\
    &\geq \frac {1}{16|\cI_f|}\sum_{U \in \cI_f}
    \1[\forall z\in U\centerdot \widetilde{(f\circ p_y)}(z) \neq \widetilde{(g\circ p_y)}(z)] dz
    \\
    &\geq \frac 1 {32} \left(1 - \frac {2\Cr(g\circ p_y)}{\Cr(f\circ p_y)}\right)
    \geq \frac 1 {32} \left(1 - \frac {2(2^{k^3+2})}{2^{k^3+4}}\right)
    \geq \frac 1 {64}.
  \end{align*}
  To finish,
  \begin{align*}
    \int_{[0,1]^d} |f(x) - g(x)|dx
    &=
    \int_{[0,1]^{d-1}} \int_{[0,1]} |(f\circ p_y)(z) - (g\circ p_y)(z)| dz dy
    \geq
    \frac {1}{64}.
  \end{align*}
\end{proofof}

Using nearly the same proof, but giving up on continuous uniform measure, it is possible
to handle other distances and more flexible target functions.

\begin{theorem}
  \label{fact:main:gen}
  Let integer $k\geq 1$
  and function $f:\R\to\R$ be given where $f$ is $(1,[0,1])$-triangle,
  and define $h:\R^d\to\R$ as $h(x) := f^k(x_1)$.
       For every $y\in\R^{d-1}$, define the affine function $p_y(z) := (z,y)$.
  Then there exist Borel probability measures $\mu$ and $\nu$ over $[0,1]^d$
  where $\nu$ is discrete uniform on $2^k+1$ points and $\mu$ is continuous and positive on exactly $[0,1]^d$
  so that
  every $g:\R^d\to\R$ with
  $\Cr(g\circ p_y) \leq 2^{k-2}$ for every $y\in\R^{d-1}$
  satisfies
  \begin{align*}
    \int |h - g|d\mu \geq \frac 1 {32},
    \qquad
    \int |\tilde h - \tilde g|d\mu \geq \frac 1 {8},
    \qquad
    \int |h - g|d\nu \geq \frac 1 {8},
    \qquad
    \int |\tilde h - \tilde g|d\nu \geq \frac 1 {4}.
  \end{align*}
\end{theorem}

\section{Limitations of depth}
\label{sec:vc}

\Cref{fact:main:gen} can be taken to say: there exists a labeling of $\Theta(2^{k^3})$ points
which is realizable by a network of depth and size $\Theta(k^3)$,
but can not be approximated by networks with depth $k$ and size $o(2^k)$.
On the other hand, this section will sketch the proof of \Cref{fact:sa_fit_few},
which implies that these $\Theta(k^3)$ depth networks realize relatively few different
labellings.
The proof is a quick consequence of the VC dimension of semi-algebraic
networks (cf. \Cref{fact:vc_sa}) and the following fact,
where $\Sh(\cdot)$ is used to denote the \emph{growth function}
\citep[Chapter 3]{anthony_bartlett_nn}.

\begin{lemma}
  \label[lemma]{fact:vc_lb}
  Let any function class $\cF$
  and any distinct points $(x_i)_{i=1}^n$   be given.
  Then with probability at least $1-\delta$ over a uniform random draw of labels
  $(y_i)_{i=1}^n$ (with $y_i \in \{-1,+1\}$),
     \[
    \inf_{f\in\cF}  \frac 1 n \sum_{i=1}^n \1[\tilde f(x_i) \neq y_i]
    \geq \frac 1 2 \left(1 - \sqrt{
        \frac {\ln(\Sh(\cF;n)) + \ln(1/\delta)}{2n}
    }\right).
              \]
\end{lemma}

The proof of the preceding result is similar to proofs of the Gilbert-Varshamov packing bound via Hoeffding's inequality
\citep[Lemma 13.5]{duchi_info_theory}.  Note that a
similar result was used by \citeauthor{warren} to prove rates of approximation of continuous functions by polynomials,
but without invoking Hoeffding's inequality \citep[Theorem 7]{warren}.

The remaining task is to control the VC dimension of semi-algebraic networks.  To this end, note
the following generalization of \Cref{fact:vc_sa}, which further provides that semi-algebraic networks
compute functions which are polynomial when restricted to certain polynomial regions.

\begin{lemma}
  \label[lemma]{fact:pp_vc}
  Let neural network graph $\fG$ be given with $\leq p$ parameters, $\leq l$ layers, and $\leq m$ total nodes,
  and suppose every gate is $(t,\alpha,\beta)$-sa.  Then
  \[
    \VC(\cN(\fG))
    \leq
        6p(l+1)\big(
      \ln(2p(l+1)) + \ln(8emt\alpha) + l\ln(\beta)
    \big).
  \]
  Additionally, given any $n\geq p$ data points, there exists a partition $\cS$ of
  $\R^p$ where each $S\in\cS$ is an intersection of predicates $\1[q \diamond 0]$ with $\diamond \in \{<,\geq\}$ and
  $q$ has degree $\leq \alpha\beta^{l-1}$, such that $F_\fG(x_i,\cdot)$ restricted to each $S\in\cS$ is a fixed polynomial of degree $\leq \beta^l$
  for every example $x_i$,
            with
   $|\cS|
   \leq
   \left(8enmt \alpha \beta^{l}\right)^{pl}$
   and
   $
   \Sh(\cN(\fG);n)
   \leq
   \left(8enmt \alpha \beta^{l}\right)^{p(l+1)}$
\end{lemma}

The proof follows the same basic structure of the VC bound for networks with piecewise polynomial
activation functions \citep[Theorem 8.8]{anthony_bartlett_nn}.  The slightly modified proof here is also very
similar to the proof of \Cref{fact:sa_crossing}, performing an induction up through the layers of the network,
arguing that each node computes a polynomial after restricting attention to some range of parameters.
The proof of \Cref{fact:pp_vc} manages to be multivariate (unlike \Cref{fact:sa_crossing}), though this requires
arguments due to \citet{warren} which are significantly more complicated than those of \Cref{fact:sa_crossing}
(without leading to a strengthening of \Cref{fact:main}).

One minor departure from the VC dimension proof of piecewise polynomial networks (cf. \citep[Theorem 8.8]{anthony_bartlett_nn})
is the following \namecref{fact:poly_counts}, which is used to track the number of regions
with the more complicated semi-algebraic networks.
Despite this generalization, the VC dimension bound is basically the same as for piecewise polynomial networks.

\begin{lemma}
  \label[lemma]{fact:poly_counts}
  Let a set of polynomials $\cQ$ be given where each $\cQ\ni q:\R^p\to\R$
  has degree $\leq \alpha$.
  Define an initial family $\cS_0$ of subsets of $\R^p$ as
    $
    \cS_0 := \big\{
      \{ a\in\R^p : q(a) \diamond 0 \}
      \ :\ q\in\cQ, \diamond \in \{ <, \geq \}
    \big\}.
  $
    Then the collection $\cS$ of all nonempty intersections of elements of $\cS_0$ satisfies
  $
        |\cS| \leq 2 \left(\frac {4e|\cQ|\alpha}{p}\right)^p.
  $
\end{lemma}

\section{Bibliographic notes and open problems}
\label{sec:bib}

Arguably the first approximation theorem of a big class by a smaller one is the Weierstrass Approximation Theorem,
which states that polynomials uniformly approximate continuous functions over compact sets \citep{weierstrass_apx}.
Refining this,
\citet{kolmogorov_width} gave a bound on how well subspaces of functions can approximate continuous functions,
and \citet{vitushkin_1,vitushkin_2} showed a similar bound for approximation by polynomials in
terms of
the polynomial degrees, dimension, and modulus of continuity of the target function.
\citet{warren} then gave an alternate proof and generalization of this result, in the process effectively
proving the VC dimension of polynomials
(developing tools still used to prove the VC dimension of neural networks \citep[Chapters 7-8]{anthony_bartlett_nn}),
and producing an analog to \Cref{fact:sa_fit_few} for polynomials.

The preceding results, however, focused on separating large classes (e.g., continuous functions of bounded modulus) from small classes (polynomials
of bounded degree).
Aiming to refine this, depth hierarchy theorems in circuit complexity separated circuits of a certain depth from circuits of a slightly
smaller depth.
As mentioned in \Cref{sec:intro}, the seminal result here is due to \citet{hastad_thesis}.
For architectures closer to neural networks, \emph{sum-product networks} (summation and product nodes) have been analyzed by
\citet{bengio_sumproduct_separation} and more recently \citet{martens_sumproduct},
and networks of linear threshold functions in 2 and 3 layers by \citet{kane_williams_ltf};
note that both polynomial gates (as in sum-product networks)
and linear threshold gates are semi-algebraic gates.
Most closely to the present
work (excluding \citep{mjt_easy_relu}, which is a vastly simplified account),
\citet{ohad_nn_apx} analyze 2- and 3-layer networks with general activation functions composed with affine mappings,
showing separations which are exponential in the input dimension.
Due to this result and also recent advances in circuit complexity \citep{rocco_apx},
it is natural to suppose \Cref{fact:main} can be strengthened to separating $k$ and $k+1$ layer networks when dimension $d$ is large;
however, none of the earlier works give a tight sense of the behavior as $d\downarrow 1$.

The triangle wave target functions considered here (e.g., cf. \Cref{fact:triangles}) have appeared in various forms throughout the literature.
General properties of piecewise affine highly oscillating functions were investigated by
\citet{szymanski_mccane_folding} and \citet{bengio_linear_regions}.
Also, \citet{schmitt_4x1x_nn} investigated the map $z\mapsto 4z(1-z)$ (as in \Cref{fact:triangles}) to show that sigmoidal networks
can not approximate high degree polynomials via an analysis similar to the one here,
however looseness in the VC bounds for sigmoidal networks prevented exponential separations and depth hierarchies.

A tantalizing direction for future work is to characterize not just one difficult function (e.g., triangle functions as in \Cref{fact:triangles}),
but many, or even all functions which are not well-approximated by smaller depths.  Arguably, this direction could have
value in machine learning, as discovery of such underlying structure could lead to algorithms to recover it.
As a trivial example of the sort of structure which could arise,
considering the following \namecref{fact:sym_rep},
stating that any symmetric signal may be repeated by pre-composing it with the ReLU triangle function.

\begin{proposition}
  \label[proposition]{fact:sym_rep}
  Set $f(z) := \srelu(2\srelu(z) - 4\srelu(z-1/2))$ (cf. \Cref{fact:triangles}),
  and let any $g:[0,1]\to[0,1]$ be given with $g(z) = g(1-z)$.
  Then $h := g\circ f^k$ satisfies $h(x) = h(x + i2^k) = g(x2^k)$ for every real $x\in [0,2^{-k}]$ and integer $i\in \{0,\ldots,2^{-k}-1\}$;
  in other words, $h$ is $2^k$ repetitions of $g$ with graph scaled horizontally and uniformly to fit within $[0,1]^2$.
\end{proposition}

\acks{%
  The author is indebted to Joshua Zahl for help navigating semi-algebraic geometry and for a simplification of
  the multivariate case in \Cref{fact:main},
  and to
  Rastislav Telg{\'a}rsky
  for an introduction to this general topic via Kolmogorov's Superposition Theorem \citep{kolmogorov_nn}.
  The author further thanks
  Jacob Abernethy,
  Peter Bartlett,
  S{\'e}bastien Bubeck,
  and Alex Kulesza
  for valuable discussions.
}

\addcontentsline{toc}{section}{References}
\ifarxiv
\bibliographystyle{plainnat}
\fi
\bibliography{nn}

\appendix

\section{Deferred proofs}
\label{sec:proofs}

This appendix collects various proofs omitted from the main text.

\subsection{Deferred proofs from \Cref{sec:sa}}

The following mechanical proof shows that standard piecewise polynomial gates,
maximization/minimization gates, and decision trees are all semi-algebraic gates.

\begin{proofof}{\Cref{fact:sa:ex}}
  \begin{enumerate}
    \item
      To start, since $\sigma:\R\to\R$ is piecewise polynomial,
      $\sigma\circ q$ can be written
      \begin{align*}
        \sigma(q(z))
        &:= p_1(q(z))\1[q(z) \diamond_1 b_1]
        + \sum_{i=2}^{t-1} p_i(q(z))\1[-q(z) \ast_{i-1} -b_{i-1}]\1[q(z) \diamond_i b_{i}]
        \\
        &\qquad+ p_t(q(z)) \1[-q(z) \ast_t -b_t]
      \end{align*}
      where for each $i\in[t]$,
      $p_i$ has degree $\leq \beta$,
      $\diamond_i \in \{<,\leq\}$,
      $\ast_i = ``<"$ when $\diamond_i = ``\leq"$ and otherwise $\ast_i = ``\leq"$,
      and $b_i\in\R$.
      As such, setting $q_i(z) := q(z) - b_1$ whenever $\diamond_i = ``<"$ and
      $q_i(z) := b_i - q(z)$ otherwise,
      it follows that $\sigma\circ q$ is $(t,\alpha,\alpha\beta)$-sa.

    \item
      Since $\min_{i\in[r]} x_i = -\max_{i\in[r]} -x_i$, it suffices to handle the maximum case,
      which has the form
      \begin{align*}
        \phmax(v)
        = \sum_{i=1}^d p_i(v)
        \left(
          \prod_{j < i}
          \1[p_i(v) > p_j(v)]
        \right)
        \left(
          \prod_{j > i}
          \1[p_i(v) \geq p_j(v)]
        \right).
      \end{align*}
      Constructing polynomials $q_{i,j} = p_j - p_i$ when $j<i$ and $q_{i,j} = p_i - p_j$ when $j > i$,
      it follows that $\phmax$ is $(r(r-1), \alpha,\alpha)$-sa.

    \item
      First consider a $k$-dt $f$, wherein the proof follows by induction on tree size.
      In the base case $k=1$, $f$ is constant.
      Otherwise, there exist functions $f_l$ and $f_r$ which are respectively $l$- and $r$-dt
      with $l + r < k$, and additionally an affine function $q_f$ so that
      \begin{align*}
        f(x)
        &= f_l(x)\1[q_f(x) < 0] + f_r(x) \1[q_f(x) \geq 0]
        \\
        &=
        \sum_{j=1}^{m_l} p_j^{(l)}(v)
        \1[q_f(x) < 0]
        \left(\prod_{i \in L_j^{(l)}} \1[q_i^{(l)}(v) < 0]\right)
        \left(\prod_{i \in U_j^{(l)}} \1[q_i^{(l)}(v) \geq 0]\right)
        \\
        &\qquad
        +
        \sum_{j=1}^{m_r} p_j^{(r)}(v)
        \1[q_f(x) \geq 0]
        \left(\prod_{i \in L_j^{(r)}} \1[q_i^{(r)}(v) < 0]\right)
        \left(\prod_{i \in U_j^{(r)}} \1[q_i^{(r)}(v) \geq 0]\right).
      \end{align*}
      where the last step expanded the semi-algebraic forms of $f_l$ and $f_r$.
      As such, by combining the sets of predicate polynomials for $f_l$ and $f_r$ together with $\{q_f\}$
      (where the former two have cardinalities $\leq l$ and $\leq r$ by the inductive hypothesis),
      and unioning together the triples for $f_l$ and $f_r$ but extending the triples to
      include $\1[q_f < 0]$ for triples in $f_l$ and $\1[q_f \geq 0]$ for triples in $f_r$,
      it follows by construction that $f$ is $(k,1,0)$-semi-algebraic.

      Now consider a $(t,k)$-bdt $g$.  By the preceding expansion, each individual tree $f_i$
      is $(k,1,0)$-sa, thus their sum is $(tk,1,0)$
      by unioning together the sets of polynomials, triples, and adding together the expansions.
  \end{enumerate}
\end{proofof}

\subsection{Deferred proofs from \Cref{sec:apx}}

The first proof shows that a collection of partitions may be refined into a single partition
whose size is at most the total number of intervals across all partitions.
As discussed in the text, while the proof has a simple idea (one need only consider boundaries
of intervals across all partitions), it is somewhat painful since there is not consistent rule for
whether specific endpoints endpoints of intervals are open or closed.

\begin{proofof}{\Cref{fact:partition_combination}}
  If $k=1$, then the result follows with $B=A=A_1$ (since all intersections are empty),
  thus suppose $k\geq 2$.
  Let $\{a_1,\ldots,a_q\}$  denote the set of distinct boundaries of intervals of $A$,
  and iteratively construct the partition $B$ as follows,
  where the construction will maintain that $B_j$ is a partition whose boundary points
  are $\{a_1,\ldots a_j\}$.
  For the base case, set $B_0 := \{ \R \}$.
  Thereafter, for every $i \in [q]$, consider boundary point $a_i$;
  since the boundary points are distinct, there must exist a single interval $U\in B_{i-1}$
  with $a_i \in U$.  $B_i$ will be formed from $B_{i-1}$ by refining $U$ in one of the following
  two ways.
  \begin{itemize}
    \item
      Consider the case that each partition $A_l$ which contains the boundary point $a_i$
      has exactly two intervals meeting at $a_i$ and moreover the closedness properties are the same,
      meaning either $a_i$ is contained in the interval which ends at $a_i$, or it is
      contained in the interval which starts at $a_i$.
      In this case, partition $U$ into two intervals so that the treatment
      of the boundary is the same as those $A_l$'s with a boundary at $a_i$.
    \item
      Otherwise, it is either the case that some $A_l$ have $a_i$ contained in the interval ending at
      $a_i$ whereas others have it contained in the interval starting at $a_i$,
      or simply some $A_l$ have three intervals meeting at $a_i$: namely, the singleton interval
      $[a_l,a_l]$ as well as two intervals not containing $a_l$.
      In this case, partition $U$ into three intervals:
      one ending at $a_i$ (but not containing it),
      the singleton interval $[a_i,a_i]$,
      and an interval starting at $a_i$ (but not containing it).
  \end{itemize}
  (These cases may also be described in a unified way: consider all intervals of $A$ which have
  $a_i$ as an endpoint, extend such intervals of positive length to have infinite length
  while keeping endpoint $a_i$ and the side it falls on,
  and then refine $U$ by intersecting it with all of these intervals,
  which as above results in either 2 or 3 intervals.)

  Note that the construction never introduces more intervals at a boundary point than
  exist in $A$, thus $|B|\leq |A| = kt$.

  It remains to be shown that a union of intersections of elements of $A$ is a union of elements of $B$.
  Note that it suffices to show that intersections of elements of $A$ are unions of elements of $B$,
  since thereafter these encodings can be used to express unions of intersections of $A$ as unions of $B$.
  As such, consider any intersection $U$ of elements of $A$; there is nothing to show if $U$ is empty,
  thus suppose it is nonempty.  In this case, it must also be an interval (e.g., since intersections of convex
  sets are convex), and its endpoints must coincide with endpoints of $A$.
  Moreover, if the left endpoint of $U$ is open, then $U$ must be formed from an intersection which
  includes an interval with the same open left endpoint, thus there exists such an interval in $A$,
  and by the above construction of $B$, there also exists an interval with such an open left endpoint in $B$;
  the same argument similarly handles the case of closed left endpoints,
  as well as open and closed right endpoints,
  namely giving elements in $B$ which match these traits.
  Let $a_r$ and $a_s$ denote these endpoints.
  By the above construction of $B$, intervals with endpoints $\{a_j,a_{j+1}\}$
  for $j\in \{r,\ldots,s-1\}$ will be included in $B$,
  and since $B$ is a partition, the union of these elements will
  be exactly $U$.  Since $U$ was an arbitrary intersection of elements of $A$,
  the proof is complete.
\end{proofof}

Next, the tools of \Cref{sec:apx:lb} (culminating in the composition rule for semi-algebraic gates (\Cref{fact:sa_to_poly}))
are used to show crossing number bounds on semi-algebraic networks
and boosted decision trees.

\begin{proofof}{\Cref{fact:sa_crossing}}
  \begin{enumerate}
    \item
      This proof first shows
      $f\circ h$ is
      $(2^i t_i \alpha_i \prod_{j \leq i-1} t_j \alpha_j \beta_j^{i-j+1} k_j, \prod_{j\leq i} \beta_j)$-poly,
      and then relaxes this expression and applies \Cref{fact:poly_cr} to obtain the desired bound.

      First consider the case $d=1$ and $h$ is the identity map, thus $f\circ h = f$.
      For convenience, set
      \[
        A_i := \prod_{j \leq i} \alpha_j,
        \quad
        B_i := \prod_{j \leq i} \beta_j,
        \quad
        C_i := \prod_{j \leq i} \beta_j^{i - j + 1} = \prod_{j \leq i} B_j,
        \quad
        M_i := \prod_{j \leq i} m_j,
        \quad
        T_i := \prod_{j \leq i} t_j.
      \]
      The proof proceeds by induction on the layers of $f$, showing that
      each node in layer $i$ is
      $(2^iT_i A_i C_{i-1} M_{i-1}, B_i)$-poly.

      For convenience, first consider layer $i=0$ of the inputs themselves:
      here, node $i$ outputs the $i^\textup{th}$ coordinate of the input,
      and is thus affine and $(1,1)$-poly.
      Next consider layer $i>0$, where the inductive hypothesis grants
      that each node in layer $i-1$ is
      $(2^{i-1}T_{i-1} A_{i-1} C_{i-2} M_{i-2}, B_{i-1})$-poly.
      Consequently, since any node in layer $i$ is $(t_i,\alpha_i,\beta_i)$-sa,
      \Cref{fact:sa_to_poly} grants it is also
      $(2^{i-1}t_i T_{i-1} A_{i-1} C_{i-2} M_{i-2} m_{i-1}(1+ \alpha_i B_{i-1}), \beta_i B_{i-1})$-poly
      as desired (since $1+\alpha_i B_{i-1} \leq 2\alpha_i B_{i-1}$).

      Next, consider the general case $d \geq 1$ and $h:\R\to\R^d$ is an affine map.
      Since every coordinate of $h$ is affine (and thus $(1,1)$-poly),
      composing $h$ with every polynomial in the semi-algebraic gates of layer 1
      gives a function $g\in \cN_1((m_i,t_i,\alpha_i,\beta_i)_{i=1}^l)$ which is equal to $f\circ h$ everywhere
      and whose gates are of the same semi-algebraic complexity.
      As such, the result follows by applying the preceding analysis to $g$.

      Lastly, the simplified terms give
      $f\circ h$ is $((2t\alpha)^l \beta^{l(l-1)/2} \prod_{j \leq l-1} m_j, \beta^{l(l+1)/2})$-poly.
      Since $\ln(\cdot)$ is strictly increasing and concave and $m_l = 1$,
      \[
        \ln\left(\prod_{j \leq l-1} m_j\right)
        = \ln\left(\prod_{j \leq l} m_j\right)
        = \sum_{j \leq l} \ln(m_j)
        \leq  l \ln(m/l)
        = \ln ( (m/l)^l ).
      \]
      It follows that $f\circ h$ is
      $((2tm\alpha/l)^l \beta^{l(l-1)/2}, \beta^{l(l+1)/2})$-poly,
      whereby the crossing number bound follows by \Cref{fact:poly_cr}.
    \item
      Given any $k$-dt $f$, the affine function evaluated at each predicate may be composed with $h$ to yield
      another affine function, thus $f\circ h :\R\to\R$ is still a $k$-dt,
      and thus $(k,1,0)$-sa by \Cref{fact:sa:ex}.
      As such, by \Cref{fact:sa_to_poly} (with $g_1(z) = z$ as the identity map),
      $f\circ h$ is $(k,0)$-poly.  (Invoking \Cref{fact:sa_to_poly} without massaging in $h$ introduces a factor $d$.)
      Similarly, for a $(t,k)$-bdt $g$, $g\circ h:\R\to\R$ is another $(t,k)$-bdt after pushing $h$ into the
      predicates of the constituent trees, thus \Cref{fact:sa:ex} grants $g\circ h$ is $(tk,1,0)$-sa,
      and \Cref{fact:sa_to_poly} grants it is $(tk(1+1),0)$-poly.
      The desired crossing number bounds follow by applying \Cref{fact:poly_cr}.
  \end{enumerate}
\end{proofof}

Next, elementary computations verify that the three functions listed in \Cref{fact:triangles} are indeed $(1,[0,1])$-triangle.

\begin{proofof}{\Cref{fact:triangles}}
  \begin{enumerate}
    \item[1-2.]
      By inspection, $f(0) = f(1) = 0$ and $f(1/2) = 1$.
      Moreover, for $x\in[0,1/2]$, $f(x) = 2x$ meaning $f$ is increasing,
      and $x\in[1/2,1]$ means $f(x) = 2(1-x)$, meaning $f$ is decreasing.
      Lastly, the properties of $g$ follow since $f=g$.

    \item[3.]
      By inspection, $h(0) = h(1) = 0$ and $h(1/2) = 1$.
      Moreover $h$ is a quadratic, thus can cross 0 at most twice, and moreover $1/2$ is the unique
      critical point (since $g'$ has degree 1), thus $g$ is increasing on $[0,1/2]$
      and decreasing on $[1/2,1]$.
  \end{enumerate}
\end{proofof}

In the case of the ReLU $(1,[0,1])$-triangle function $f$ given in \Cref{fact:triangles},
the exact form of $f^k$ may be established as follows.  (Recall that this refined form
allows for the use of Lebesgue measure in \Cref{fact:main},
and also the repetition statement in \Cref{fact:sym_rep}.)

\begin{proofof}{\Cref{fact:triangle:relu:shape}}
  The proof proceeds by induction on the number of compositions $l$.
  For the base case $l=1$,
  \[
    f^1(z) = f(z) = \begin{cases}
      2z
      &\textup{when $z\in[0,1/2]$},
      \\
      2(1-z)
      &\textup{when $z\in(1/2,1]$},
      \\
      0
      &\textup{otherwise}.
    \end{cases}
  \]
  For the inductive step,
  first note for any $z\in[0,1/2]$,
  by symmetry of $f^l$ around 1/2 (i.e., $f^l(z) = f^l(1-z)$ by the inductive hypothesis),
  and by the above explicit form of $f^1$,
  \[
    f^{l+1}(z) = f^l(f(z)) = f^l(2z) = f^l(1-2z) = f^l(f(1/2-z)) = f^l(f(z+1/2)) = f^{l+1}(z+1/2),
  \]
  meaning the case $z \in (1/2,1]$ is implied by the case $z\in[0,1/2]$.
  Since the unique nonnegative integer $i_{l+1}$ and real $z_{l+1}\in [0,1)$ satisfy
  $2z = 2(i_{l+1} + z_{l+1})2^{-l-1} = (i_{l+1} + z_{l+1})2^{-l}$,
  the inductive hypothesis grants
  \[
    (f^l \circ f)(z) = f^l(2z) = \begin{cases}
      2z_{l+1}
      &\textup{when $0 \leq z_{l+1} \leq 1/2$},
      \\
      2(1 - z_{l+1})
      &\textup{when $1/2 < z_{l+1} < 1$},
    \end{cases}
  \]
  which completes the proof.
\end{proofof}

The proof of the slightly more general form of \Cref{fact:main}
is as follows;
it does not quite imply \Cref{fact:main}, since the constructed measure is not the Lebesgue measure even for
the ReLU-based $(1,[0,1])$-triangle function from \Cref{fact:triangles}.

\begin{proofof}{\Cref{fact:main:gen}}
  First note some general properties of $f^k$.
  By \Cref{fact:mono:gen:ub:2},
  $f^k$ is $(2^{k-1}, [0,1])$-triangle,
  which means there exist $s:= 2^{k}+1$ points $(z_i)_{i=1}^s$
  so that $f^k(z_i) = \1[\textup{$i$ is odd}]$,
  and moreover $f^k$ is continuous and equal to $1/2$ at exactly
  $2^k$ points (by the strict increasing/decreasing part of the triangle wave definition),
  which is a finite set of points and thus has Lebesgue measure zero.
  Taking $p_y:\R\to\R^{d}$ to be the map $p_y(z) = (z,y)$ where $y\in\R^{d-1}$,
  then $(h\circ p_y)(z) = h((z,y)) = f^k(z)$,
  thus letting $\cI$ denote the $2^k$ pieces within which $\widetilde{f^k}$ is constant,
  it follows that $\widetilde{h\circ p_y}$ is constant within the same set of pieces
  and thus $\Cr(h\circ p_y) = s$.

  Now consider the discrete case, where $\nu$ denotes the uniform
  measure over the $s$ points $(x_i)_{i=1}^s$ defined as $x_i := p_0(z_i)\in\R^d$.
  Further consider the two types of distance.
  \begin{itemize}
    \item
      Since $z_i < z_{i+1}$ and $\widetilde{f^k}(z_i) \neq \widetilde{f^k}(z_{i+1})$,
      then taking $(U_i)_{i=1}^{s}$ to denote the intervals of $\cI$ sorted by their left endpoint,
      $z_i \in U_i$ for $i \in [s]$.
      By \Cref{fact:mono:crossing_lb:2},
      \begin{align*}
        \int |\tilde h - \tilde g|d\nu
        &= \frac 1 {s} \sum_{i=1}^{s} | \tilde h(x_i) - \tilde g(x_i) |
        = \frac 1 {s} \sum_{i=1}^{s} | \widetilde{f^k}(z_i) - \widetilde{g\circ p_0} (z_i) |
        \\
        &\geq \frac 1 {s} \sum_{i=1}^{s} \1[ \forall z \in U_i
        \centerdot \widetilde{f^k}(z) \neq \widetilde{g\circ p_0}(z) ]
        \\
        &\geq \frac 1 2\left(1 - 2 \left(\frac {2^{k-2}}{s}\right)\right)
        \geq \frac 1 4.
      \end{align*}

    \item
      Since $f^k(z_i) \in \{0,1\}$, then $\widetilde{f^k}(z_i) \neq \widetilde{g}(x_i)$ implies
      $|f^k(z_i) - g(x_i)| \geq 1/2$,
      thus $\int_{[0,1]^d} |h - g|d\nu \geq \int_{[0,1]^d} |\tilde h -\tilde g|d\nu/2 \geq 1/8$.
  \end{itemize}

  Construct the continuous measure $\mu$ as follows,
  starting with the construction of a univariate measure $\mu_0$.
  Since $f^k$ is continuous, there exists a $\delta\in (0, \min_{i\in[s-1]} |z_i - z_{i+1}|/2)$
  so that $|f^k(z) - f^k(z_i)|\leq 1/4$ for any $i\in[s]$ and $z$ with $|z-z_i|\leq \delta$.
  As such, let $\mu_0$ denote the probability measure which places half of its mass
  uniformly on these $s$ balls of radius $\delta$ (which must be disjoint since $f^k$ alternates between
  0 and 1 along $(z_i)_{i=1}^s$),
  and half of its mass uniformly on the remaining subset of $[0,1]$.
  Finally, extend this to a probability measure $\mu$ on $[0,1]^d$ uniformly, meaning
  $\mu$ is the product of $\mu_0$ and the measure $\mu_1$ which is uniform over $[0,1]^{d-1}$.
  Now consider the two types of distances.
  \begin{itemize}
    \item
      By \Cref{fact:mono:crossing_lb:2},
      \begin{align*}
        \int |\tilde h - \tilde g|d\mu(x)
        &=
        \iint |\widetilde{f^k}(p_y(z)) - \tilde g(p_y(z))|d\mu_0(z) d\mu_1(y)
        \\
        &=
        \int \sum_{U \in \cI} \int \1[z\in U
        \land \widetilde{f^k}(z)) \neq \tilde g(p_y(z))] d\mu_0(z) d\mu_1(y)
          \\
        &\geq
          \int \frac 1 {2s}\sum_{U \in \cI}  \1[\forall z\in U\centerdot
          \widetilde{f^k}(z)) \neq \widetilde{g\circ p_y}(z)] d\mu_1(y)
          \\
          &\geq \frac 1 4\left(1 - 2 \left(\frac {2^{k-2}}{s}\right)\right)
          \geq \frac 1 8.
      \end{align*}

    \item
      For any $y\in\R^{d-1}$ and $U_i\in \cI$ (with corresponding $z_i \in U_i$),
      if $\widetilde{f^k}(z) \neq \widetilde{g\circ p_y}(z)$ for every $z\in U_i$,
      then
      \[
        \int_{U_i} | {f^k}(z) - g(p_y(z))|d\mu_0(z)
        \geq
        \int_{|z-z_i|\leq \delta} | {f^k}(z) - 1/2|d\mu_0(z)
        \geq
        \frac 1 4 \mu_0(\{ z\in U_i : |z-z_i|\leq \delta\})
        \geq
        \frac 1 {8s}.
      \]
      By \Cref{fact:mono:crossing_lb:2},
      \begin{align*}
        \int |h - g|d\mu(x)
        &=
        \iint |h(p_y(z)) - g(p_y(z))|d\mu_0(z) d\mu_1(y)
        \\
        &\geq
        \int \sum_{U \in \cI}
        \1[\forall z\in U\centerdot \widetilde{f^k}(z) \neq \tilde g(p_y(z))]
        \int_U | {f^k}(z) - g(p_y(z))|d\mu_0(z)
        d\mu_1(y)
        \\
        &\geq
        \int \frac 1 {8s}\sum_{U \in \cI}
        \1[\forall z\in U\centerdot \widetilde{f^k}(z) \neq \widetilde{g\circ p_y}(z) ]
        d\mu_1(y)
        \\
        &\geq \frac 1 {16}\left(1 - 2 \left(\frac {2^{k-2}}{s}\right)\right)
        \geq \frac {1}{32}.
      \end{align*}
  \end{itemize}
\end{proofof}

As a closing curiosity,
\Cref{fact:main:gen} implies the following statement regarding polynomials.

\begin{corollary}
  \label[corollary]{fact:poly_apx}
  For any integer $k\geq 1$, there exists a polynomial $h:\R^d\to\R$ with degree $2^k$
  and a corresponding continuous measure $\mu$ which is positive everywhere over $[0,1]^d$
  so that every polynomial $g:\R^d\to\R$ of degree $\leq 2^{k-3}$ satisfies
  $
    \int |h-g|d\mu \geq 1/32.  $
\end{corollary}

  \begin{proof}
  Set $f(z) = 4z(1-z)$, which by \Cref{fact:triangles} is $(1,[0,1])$-triangle,
  thus $f^k$ is $(2^{k-1},[0,1])$-triangle with $\Cr(f^k) = 2^k+1$ by \Cref{fact:mono:gen:ub:2},
  and $f^k$ has degree $2^k$ directly; thus set $h(x) = f^k(x_1)$.
  Next, for any polynomial
  $g:\R^d\to\R$ of degree $\leq 2^{k-3}$,
  $g \circ p_y :\R\to\R$ is still a polynomial of degree $\leq 2^{k-3}$
  for every $y\in\R^{d-1}$ (where $p_y(z) = (z,y)$ as in \Cref{fact:main:gen}),
  and so \Cref{fact:poly_cr} grants
  $\Cr(g\circ p_y)\leq 1+2^{k-3} \leq 2^{k-2}$.
  The result follows by \Cref{fact:main:gen}.
\end{proof}

\subsection{Deferred proofs from \Cref{sec:vc}}

First, the proof of a certain VC lower bound which mimics the Gilbert-Varshamov bound;
the proof is little more than a consequence of Hoeffding's inequality.

\begin{proofof}{\Cref{fact:vc_lb}}
        For convenience, set $m:=\Sh(\cF;n)$,
  and let $(a_1,\ldots,a_m)$ denote these dichotomies (meaning $a_j \in \{0,1\}^n$),
  and with foresight set
    $
    \epsilon
        := \sqrt{\ln(m/\delta)/(2n)}
    $.
       Let $(Y_i)_{i=1}^n$ denote fair Bernoulli random labellings for each point,
  and note by symmetry of the fair coin that for any fixed dichotomy $a_j$,
  \begin{align*}
    \Pr\left[\frac 1 n \sum_{i=1}^n |(a_j)_i - Y_i| < 1/2 - \epsilon\right]
    = \Pr\left[\frac 1 n \sum_{i=1}^n Y_i < 1/2 - \epsilon\right].
  \end{align*}
  Consequently, by a union bound over all dichotomies and lastly by Hoeffding's inequality,
  \begin{align*}
    \Pr\left[\exists f\in \cF \centerdot \frac 1 n \sum_{i=1}^n |\tilde f(x_i) - Y_i| < 1/2 - \epsilon\right]
    &\leq
    \sum_{j=1}^m
    \Pr\left[\frac 1 n \sum_{i=1}^n |(v_j)_i - Y_i| < 1/2 - \epsilon\right]
    \\
    &=
    m
    \Pr\left[\frac 1 n \sum_{i=1}^n Y_i < 1/2 - \epsilon\right]
    \\
    &\leq
    m \exp(-2n\epsilon^2)
    \leq \delta,
  \end{align*}
  where the last step used the choice of $\epsilon$.
\end{proofof}

The remaining deferred proofs do not exactly follow the order of \Cref{sec:vc},
but instead the order of dependencies in the proofs.  In particular, to control
the VC dimension, first it is useful to prove \Cref{fact:poly_counts},
which is used to control the growth of numbers of regions as semi-algebraic gates are combined.

\begin{proofof}{\Cref{fact:poly_counts}}
  Fix some ordering $(q_1, q_2,\ldots,q_{|\cQ|})$ of the elements of $\cQ$,
  and for each $i \in [|\cQ|]$ define two functions $l_i(a) := \1[q_i(a) < 0]$ and $u_i(a) := \1[q_i(a) \geq 0]$,
  as well as two sets $L_i := \{a\in\R^p : l_i(a) = 1\}$ and $U_i := \{a\in\R^p : u_i(a) = 1\}$.
  Note that
  \[
    \cS := \Big\{ (\cap_{i \in A} L_i) \cap (\cap_{i \in B}) : A \subseteq [|\cQ|], B \subseteq [|\cQ|]\ \Big\} \setminus \{\emptyset\}.
  \]
  Additionally consider the set of sign patterns
  \[
    V := \left\{
      \left(l_1(a), u_i(a), \ldots, l_{|\cQ|}(a), u_{|\cQ|}(a)\right)
      :
      a\in\R^p
    \right\}.
  \]
  Distinct elements of $\cS$ correspond to distinct sign patterns in $V$:
  namely, for any $C\in\cS$, using the ordering of $\cQ$ to encode $A$ and $B$
  as binary vectors of length $|\cQ|$, the corresponding interleaved binary vector of length $2|\cQ|$
  is distinct for distinct choices of $(A,B)$.
  (For each $i$ that appears in neither $A$ nor $B$, there two possible encodings in $V$:
  having both coordinates corresponding to $i$ set to 1, and having them set to 0.
  On the other hand, a more succinct encoding based just on $(l_i)_{i=1}^{|\cQ|}$ fails
  to capture those sets arising from intersections of proper subsets of $\cQ$.)
  As such,
  making use of growth function bounds for sets of polynomials \citep[Theorem 8.3]{anthony_bartlett_nn},
  \[
    |\cS| \leq |V| \leq 2 \left(\frac {4e\alpha|\cQ|}{p}\right)^p.
  \]
\end{proofof}

Thanks to \Cref{fact:poly_counts}, the proof of the VC dimension bound \Cref{fact:pp_vc}
follows by induction over layers, effectively keeping track of a piecewise (regionwise?) polynomial
function as with the proof of \Cref{fact:sa_crossing} (but now in the multivariate case).

\begin{proofof}{\Cref{fact:pp_vc}}
  First note that this proof follows the scheme of a VC dimension proof for networks with piecewise
  polynomial activation functions \citep[Theorem 8.8]{anthony_bartlett_nn},
  but with \Cref{fact:poly_counts} allowing for the more complicated semi-algebraic gates,
  and some additional bookkeeping for the (semi-algebraic) shapes of the regions of the partition $\cS$.

  Let examples $(x_j)_{j=1}^n$ be given with $n \geq p$,
  let $m_i$ denote the number of nodes in layer $i$ (whereby $m_1+\cdots+m_l = m$),
  and let $f := F_\fG : \R^p \times \R^d \to \R$ denote the function evaluating the neural network (as in \Cref{sec:sa:nn}),
  where the two arguments are the parameters $w\in\R^p$ and the input example $x\in\R^d$.
  The goal is to upper bound the number of dichotomies
  \[
    K := \Sh(\cN(\fG);n) = \left|
    \left\{
      (\sgn(f(w,x_1)), \ldots, \sgn(f(w,x_n)))
      :
      w\in\R^p
    \right\}
    \right|.
  \]
  The proof will proceed by producing a sequence of partitions $(\cS_i)_{0=1}^l$ of $\R^p$ and
  two corresponding sequences of sets of polynomials $(\cP_i)_{i=0}^l$ and $(\cQ_i)_{i=0}^l$
  so that for each $i$,
  $\cP_i$ has polynomials of degree at most $\beta^i$,
  $\cQ_i$ has polynomials of degree at most $\alpha\beta^{i-1}$,
  and over any parameters $S\in\cS_i$,
  there is an assignment of elements of $\cP_i$ to nodes of layer $i$
  so that for each example $x_j$, every node in layer $i$ evaluates the corresponding fixed
  polynomial in $\cP_i$;
  lastly, the elements of $\cS_i$ are intersections of sets of the form
  $\{w\in\R^p : q(w) \diamond 0 \}$ where $q \in \cQ_i$ and $\diamond \in \{<,\geq\}$,
  and the partition $\cS_{i+1}$ refines $\cS_i$ for each $i$ (meaning for each $U \in \cS_{i+1}$ there exists $S\supseteq U$ with $S\in \cS_i$).
  Setting the final partition $\cS := \cS_l$,
  this in turn will give an upper bound on $K$,
  since the final output within each element of $\cS$ is a fixed polynomial of degree at most $\beta^l$,
  whereby the VC dimension of polynomials \citep[Theorem 8.3]{anthony_bartlett_nn} grants
  \begin{align}
    K
    \leq \sum_{S \in \cS}
    \left|
    \left\{
      (\sgn(f(w,x_1)), \ldots, \sgn(f(w,x_n)))
      :
      w\in S
    \right\}
    \right|
    \leq
    2 |\cS| \left(
      \frac
      {2en\beta^l}{p}
    \right)^p.
    \label{eq:pp_vc:1}
  \end{align}

  To start, consider layer 0 of the input coordinates themselves,
  a collection of $d$ affine maps.
  Consequently, it suffices to set $\cS_0 := \{ \R^p \}$,
  $\cQ_0 := \emptyset$, and $\cP_0$ to be the $nd$ possible coordinate maps corresponding to
  all $d$ coordinates of all $n$ examples.

  For the inductive step, consider some layer $i+1$.
  Restricted to any $S\in\cS_i$, the nodes of the previous layer $i$ compute fixed polynomials
  of degree $\beta^i$.
  Each node in layer $i+1$ is $(t,\alpha,\beta)$-sa,
  meaning there are $t$ predicates, defined by polynomials of degree $\leq\alpha$, which define regions
  wherein this node is a fixed polynomial.
  Let $Q_S$ denote this set of predicates,
  where $|Q_S|\leq tnm_{i+1}$ by considering the $n$ possible input examples and the $t$ possible predicates
  encountered in each of the $m_{i+1}$ nodes in layer $i+1$,
  and set
  $ 
    Q_{i+1} := Q_i \bigcup \left(\cup_{S\in \cS_i} Q_S\right).
  $
  By the definition of semi-algebraic gate, each node in layer $i+1$ computes a fixed polynomial when
  restricted to a region defined by an intersection of predicates which moreover are defined by $Q_{i+1}$.
  As such, defining $\cS_{i+1}$ as the refinement of $\cS_{i+1}$ which partitions each $S\in\cS_i$ according
  to the intersections of predicates encountered in each node,
  then \Cref{fact:poly_counts} on each $Q_S$ grants
  \begin{equation}
    |\cS_{i+1}|
    \leq \sum_{S \in \cS_i} |\{\textup{all nonempty intersections of $Q_S$}\}|
    \leq
    2|\cS_i| \left(\frac {4enm_{i+1}t \alpha \beta^{i}}{p}\right)^p,
    \label{eq:pp_vc:2}
  \end{equation}
  which completes the inductive construction.

  The upper bound on $K$ may now be estimated.
  First, $|\cS|$ may be upper bounded by applying \cref{eq:pp_vc:2} recursively:
  \begin{align*}
    |\cS|
    \leq
    |\cS_0|
    \prod_{i=1}^l
    \left(\frac {8enm_{i}t \alpha \beta^{i-1}}{p}\right)^p
    \leq
    \left(8enmt \alpha \beta^{l-1}\right)^{pl}.
  \end{align*}
  Continuing from \Cref{eq:pp_vc:1},
  \begin{align*}
    K
    &\leq
    2 |\cS| \left(
      \frac
      {2em\beta^l}{p}
    \right)^p
                         \leq
    \left(8enmt \alpha \beta^{l}\right)^{p(l+1)}.
  \end{align*}
  \\
  To compute $\VC(\cN(\fG))$, it suffices to find $N$ such that $\Sh(\cN(\fG);N) < 2^N$,
  which in turn is implied by $p(l+1)\ln(N) + p(l+1)\ln(8emt\alpha\beta^l) < N\ln(2)$.
  Since $\ln(N) = \ln(N/(2p(l+1)) + \ln(2p(l+1)) \leq N/(2p(l+1)) - 1 + \ln(2p(l+1))$
  and $\ln(2) - 1/2 > 1/6$,
  it suffices to show
  \[
    6p(l+1)\left(\ln(2p(l+1)) + \ln(8emt\alpha\beta^l)\right) \leq N.
  \]
  As such, the left hand side of this expression is an upper bound on $\VC(\cN(\fG))$.
\end{proofof}

The proofs of \Cref{fact:vc_sa} and \Cref{fact:sa_fit_few} from \Cref{sec:intro}
are now direct from \Cref{fact:pp_vc} and \Cref{fact:vc_lb}.

\begin{proofof}{\Cref{fact:vc_sa}}
  This statement is the same as \Cref{fact:pp_vc} with some details removed.
\end{proofof}

\begin{proofof}{\Cref{fact:sa_fit_few}}
  By the bound on $\Sh(\cN(\fG);n)$ from \Cref{fact:pp_vc},
  \begin{align*}
    n = \frac n 2 + \frac n 2
    &\geq 2\ln(1/\delta) + 4 pl^2 \ln(8emt\alpha\beta p(l+1)) + \frac n 2
    \\
    &\geq 2\ln(1/\delta) + 2 p(l+1) \ln(8emt\alpha\beta^l) + 2p(l+1)\left( \ln(p(l+1))) + \frac {n}{2p(l+1)} - 1 \right)
    \\
    &\geq 2\ln(1/\delta) + 2 p(l+1) \ln(8emt\alpha\beta^l) + 2p(l+1)\ln(n)
    \\
    &\geq 2\ln(1/\delta) + 2\ln(\Sh(\cN(\fG);n)).
  \end{align*}
  The result follows by plugging this into \Cref{fact:vc_lb}.
\end{proofof}

\subsection{Deferred proofs from \Cref{sec:bib}}

\begin{proofof}{\Cref{fact:sym_rep}}
  Immediate from \Cref{fact:triangle:relu:shape}.
\end{proofof}

\end{document}